\documentclass[runningheads]{llncs}

\usepackage{eccv}

\usepackage{eccvabbrv}

\usepackage{graphicx}
\usepackage{booktabs}
\usepackage{dsfont}
\usepackage{booktabs}
\usepackage{algorithm}
\usepackage{algorithmic}

\usepackage[accsupp]{axessibility}  

\usepackage{hyperref}

\usepackage{orcidlink}

\begin{document}

\title{Spread them Apart: Towards Robust Watermarking of Generated Content} 

\titlerunning{Spread them Apart}

\author{Mikhail Pautov \inst{1}  
\and
Danil Ivanov \inst{2} \and 
Andrey V. Galichin \inst{3,4} \and
Oleg Rogov \inst{4} \and
Ivan Oseledets \inst{5}}

\authorrunning{Mikhail Pautov et al.}

\institute{AXXX
\and
VeinCV 
\and
Applied AI Institute 
\and 
MTUCI 
\and 
Institute of Numerical Mathematics\\}

\maketitle

\begin{abstract}
  Generative models that can produce realistic images have improved significantly in recent years. The quality of the generated content has increased drastically, so sometimes it is very difficult to distinguish between the real images and the generated ones. Such an improvement comes at a price of ethical concerns about the usage of the generative models: the users of generative models can improperly claim ownership of the generated content protected by a license.  In this paper, we propose an approach to embed watermarks into the generated content to allow future detection of the generated content and identification of the user who generated it. The watermark is embedded during the inference of the model, so the proposed approach does not require the retraining of the latter. We prove that watermarks embedded are guaranteed to be robust against additive perturbations of a bounded magnitude. We apply our method to watermark diffusion models and show that it matches state-of-the-art watermarking schemes in terms of robustness to different types of synthetic watermark removal attacks.
\end{abstract}

\section{Introduction}
\label{sec:intro}
Recent advances in generative models have brought the performance of image synthesis tasks to a whole new level. For example, the quality of the images generated by diffusion models \cite{croitoru2023diffusion,rombach2022high,esser2024scaling} is now comparable to that of human-drawn pictures or photographs. Compared to generative adversarial networks  \cite{goodfellow2014generative,brock2019large},  diffusion models allow the generation of high-resolution, diverse, and naturally looking pictures. More than that, the generation process with diffusion models is more stable, controllable, and explainable. They are easy to use and are widely deployed as tools for data generation, image editing \cite{kawar2023imagic,yang2023paint}, music generation \cite{schneider2024mousai}, text-to-image synthesis \cite{saharia2022photorealistic,zhang2023adding,ruiz2023dreambooth}, and in other multimodal settings. 

Unfortunately, several ethical and legal issues may arise from the usage of diffusion models.  On the one hand, since diffusion models can be used to generate fake content, for example, deepfakes \cite{zhao2021multi,narayan2023df}, it is crucial to develop automatic tools to verify that a particular digital object is artificially generated. On the other hand, a dishonest user of the model protected by a copyright license can query it, receive the result of generation, and later claim exclusive copyright, possibly disobeying the copyright agreement. In this work, we focus on the detection of the images generated by a particular diffusion model and on the identification of the end-user who queried the model to generate an image. We develop a technique to embed the digital watermarks into the generated content during the inference of the generative model. Namely, during generation, we optimize the latent representation of the image so that the output image satisfies the set of predefined end-user-specific inequalities. These inequalities are later used to reliably identify the end-user, so the method solves both content detection and user attribution tasks. 
The detailed description of the approach is presented in the upcoming sections.  

\begin{figure*}[tb]
    \centering
    \includegraphics[width=\textwidth]{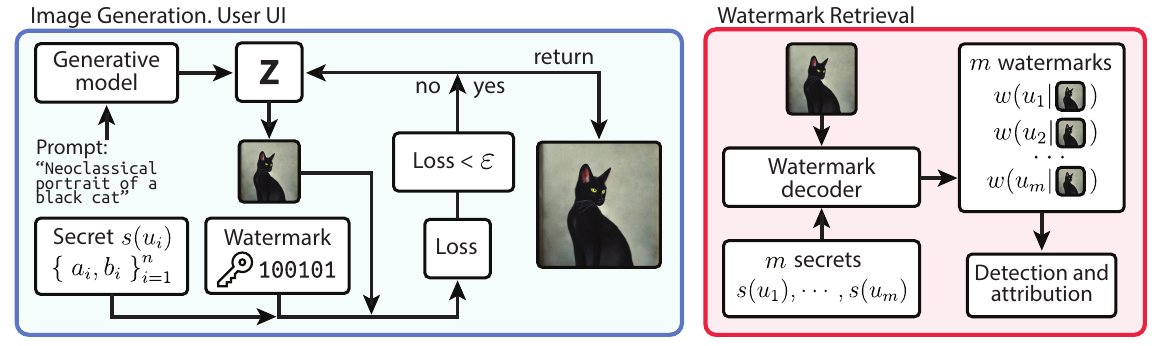}
    \caption{Illustration of the proposed method. During the image generation phase, the user $u_i$ queries the model with the prompt. Given the prompt, the model produces the latent $z$, from which the image is generated. If the image generated satisfies the constraint $\mathcal{L}_{wm} < \varepsilon$ (meaning the watermark is successfully embedded), it is yielded to the user; otherwise, the loss function from Eq. \ref{eq:TotalLoss} is minimized with respect to the latent $z$. Note that the value of $\varepsilon$ may vary from image to image. During the watermark retrieval phase, given the image $x$ and $m$ secrets, $s(u_1), \dots, s(u_m)$, the watermark decoder extracts m watermarks, $w(u_1|x), \dots, w(u_m|x)$. Then, the image is attributed to the user $u$ according to the Eq. \ref{eq:attribution}. }
    \label{fig:big_teaser}
\end{figure*}

Our contributions  are threefold: 

\begin{itemize}
    \item We propose \emph{Spread them Apart}, the framework to embed digital watermarks into the generative content of continuous nature. Our method embeds the watermark during the process of content generation and, hence, does not require additional training of the generative model. 
    \item We apply the framework to watermark images generated by a diffusion model and prove that the watermarks embedded are guaranteed to be robust to the additive perturbations of a bounded magnitude, multiplicative perturbations, and exponentiations.  

    \item Experimentally, we show that our approach yields watermarks that are robust to different types of post-processing of the images aimed at watermark removal, such as brightness and contrast adjustment, gamma correction, or a white-box adversarial attack. 
\end{itemize}

\section{Related Work}

\subsection{Diffusion Models}

Inspired by non-equilibrium statistical physics, \cite{sohldickstein2015deepunsupervisedlearningusing} introduced the diffusion model to fit complex probability distributions. Later, \cite{ho2020denoisingdiffusionprobabilisticmodels} introduced a new class of models called Denoising Diffusion Probabilistic Models (DDPM) by establishing a novel connection between the diffusion model and the denoising score matching. Shortly after, the Latent Diffusion Model (LDM)  was developed to improve efficiency and reduce computational complexity, with the diffusion process conducted in a latent space   \cite{rombach2022high}. During training, the LDM uses an encoder $\mathcal{E}$ to map an input image $x$ to the latent space, $z = \mathcal{E}(x)$. For the reverse operation, a decoder $\mathcal{D}$ is employed, so that $\hat{x} = \mathcal{D}(z)$. During inference, the LDM initializes a noise vector $z \sim \mathcal{N}(0, I)$ in the latent space and iteratively denoises it. The decoder then maps the final latent representation back to the image space.

\subsection{Watermarking of Digital Content}
Watermarking has been recently adopted to protect the intellectual property of neural networks \cite{wu2020watermarking,pautov2024probabilistically} and generated content \cite{kirchenbauer2023watermark,zhaoprovable,fu2024watermarking}. In a nutshell, watermarking of generated content is done by injecting digital information within the generated content, allowing the subsequent extraction.  Existing methods of digital content watermarking can be divided into two categories: content-level watermarking and model-level watermarking. The methods of content-level watermarking operate within a particular representation of content, for example, in the frequency domain of the image signal \cite{o1996watermarking,cox1996secure}. Notably, when the image is manipulated in the frequency domain, the watermark embedding process can be adapted to produce watermarks that are robust to geometrical image transformations, such as rotations and translations \cite{wen2024tree}. In contrast, model-level watermarking approaches are designed to embed information during the generation process. In end-to-end methods, the models to embed and extract watermark are trained jointly \cite{zhu2018hidden,hayes2017generating}. In \cite{yu2021artificial}, it was proposed to teach the watermark encoder on the training data of the generative model; such an approach yields a watermarking scheme that is conditioned on the generative model and its training dataset. This method was later adapted to latent diffusion models \cite{fernandez2023stable} and unconditional diffusion models \cite{zhao2023recipe}. In contrast, there are methods that do not require additional model training. These methods are designed to alter the output distribution of the generative model to embed previously learned watermark into the model or the content itself \cite{kirchenbauer2023watermark,wen2024tree}.

\subsection{Robustness to Watermark Removal Attacks}
Watermark removal attacks are aimed at removing the watermark embedded into the model's weights or generated content. In the prior works on removing the watermarks from generated images \cite{li2019towards,cao2019generative}, the attack problem is formulated in terms of the image-to-image translation task, and methods to remove watermarks via an auxiliary generative adversarial network are presented. Other approaches \cite{hertz2019blind,liang2021visible,sun2023denet} perform watermark removal in two steps: firstly, the visual watermark is localized within an image; secondly, it is removed via a multi-task learning framework. In practice, a watermarking scheme has to be robust to destructive and constructive attacks, or synthetic transformations of the data. Destructive transformations,  such as brightness and contrast adjustment, geometric transformations, such as rotations and translations, compression methods, and additive noise, are aimed at watermark removal by applying a transformation. In contrast, constructive attacks treat watermarks as noise and are aimed at the restoration of original content \cite{zhang2024robust}. It is usually done by applying purification techniques, such as Gaussian blur \cite{hosam2019attacking} or image inpainting \cite{liu2021wdnet,xu2017automatic}.

\section{Problem Statement}
In this section, we formulate the problem statement and the research objectives. Note that we focus on the watermarking of images generated by diffusion models, but the formulation below is valid for watermarking of any generated content, for example, audio, video, or text.  

\subsection{Image Watermarking}
In our approach, we focus on  \emph{detection} and \emph{attribution} of the generated image simultaneously: while detection is aimed at verifying whether a particular image is generated by a given model,  attribution is aimed at determining the user who generated the image. 

Suppose that we are given the generative model $f$ deployed in the black-box setting, i.e., as a service: in the generation phase, a user $u_i \in [u_1, \dots, u_m]$ sends a query to the model and receives a generated image $x \in \mathbb{R}^d$. If $x$ is a watermarked image, the owner of model $f$ should be able to identify that $x$ is generated by user $u_i$ by querying the model $f$. In our method, the image is watermarked during the \emph{generation} phase, not during the post-processing. We formulate the process of watermarking and attribution in the following way: 
\begin{enumerate}
    \item When the user $u_i \in [u_1, \dots, u_m]$ registers in the service, it is assigned a pair of \emph{public} and \emph{private} keys, namely, the watermark $w(u_i)$ and the secret $s(u_i)$. Watermark is a binary string of length $n$ and the secret is the sequence of tuples of length $n$, where each tuple is a pair of unique positive numbers treated as indices: $w(u_i) \in \{0,1\}^n, \ s(u_i) \in \mathbb{Z}_{+}^{2n}.$ 
    \item When the user $u_i$ queries the model $f,$ it generates the image $x$ with the watermark $w(u_i)$ embedded into $x$.
    \item  When the watermarked object $x$ is received by the model owner, it extracts the watermark $w(u_i|x)$ using the secret $s(u_i)$ of the user $u_i$  and compares it with the watermark $w(u_i)$ assigned to the user $u_i$. To do so, we  compute the bitwise distance $d(w(u_i|x), w(u_i))$ between $w(u_i|x)$ and $w(u_i)$:
    \begin{equation}
        d(w(u_i|x), w(u_i)) = \sum_{j=1}^n \mathds{1} (w(u_i|x)_j \ne w(u_i)_j).
    \end{equation}
\end{enumerate}
    \begin{remark}
    For robustness to watermark removal attack, in case of a single user $u_i$, we flag the object $x$ as generated by the user $u_i$ if the distance $d(w(u_i|x), w(u_i))$ is either small or large, namely, if 
    \begin{equation}
    \label{eq:att_rule}
        d(w(u_i|x), w(u_i)) \in [0, \tau_1] \cup [\tau_2, n], 
    \end{equation}
    where $\tau_1 \ll n$ and $\tau_2 \gg 0.$
    This procedure is known as the double-tail detection \cite{jiang2023evading} and is effective against watermark flipping attacks. 
    \end{remark}

\subsection{The Probability of Incorrect Attribution}
We assume that the watermark $w(u_i)$ attributed to the user $u_i$ is drawn randomly and uniformly from the set of all possible $n-$bit watermarks, $\{0,1\}^n$. We formulate the detection problem as a hypothesis test. In case of a single user $u_i$, we define the null hypothesis $\mathcal{H}_0 =$ ``the object $x$ is generated not by $u_i$'' and the  alternative hypothesis $\mathcal{H}_1 =$ ``the object $x$ is generated  by $u_i$''. Additionally, under the null hypothesis, we assume that the $j'$th bit in the watermark $w(u_i|x)$ extracted from $x$ is the same as the $j'$th bit from $w(u_i)$ with the probability $p_i$. 

In the case of a single user $u_i$ and given the attribution rule from the Eq. \ref{eq:att_rule}, we compute the probability of the false attribution, namely, 
\begin{align}
\label{eq:fpr_1}
    & FRP(1)|_{u_i} = \mathbb{P} \left[d(w', w(u_i)) \in [0, \tau_1] \cup [\tau_2, n] \right] = \nonumber \\ &\sum_{q\in [0, \tau_1] \cup [\tau_2, n]}{n \choose q} p_i^q (1-p_i)^{n-q},
\end{align}
where $w' = w(u_i|x)$ is a random watermark uniformly sampled from $\{0,1\}^n$, namely, ${w' \sim \{0,1\}^n, w' \ne w(u_i)}.$

In case of $m$ users,  the probability $FPR(m)$ of incorrect attribution of the non-watermarked image $x$ to some other user $u_j \in [u_1, \dots, u_m]$  is upper bounded by the probability below:
\begin{equation}
\begin{aligned}
    \label{eq:fpr_m}
    FPR(m) & \le  \ \mathbb{P}_{w' \sim \{0,1\}^n} [ \exists u_j \in [u_1,  \dots, u_m]:  d(w', w(u_j)) \in [0,\tau_1] \cup [\tau_2, n]] \le \\ & \le \sum_{u_j \in [u_1,\dots, u_m]} FPR(1)|_{u_j} = \hat{p}.
\end{aligned}
\end{equation}
Note that this upper bound holds regardless of the independence of random variables $\xi_1, \dots, \xi_m,$ where 
\begin{equation}
\label{eq:indicator}
    \xi_i = \mathds{1}[d(w(u_i|x), w(u_i)) \in [0,\tau_1] \cup [\tau_2, n]].
\end{equation}

\begin{remark}
In our experiments, the probabilities $p_i$ from above are estimated to be close to $\frac{1}{2}.$
\end{remark}

\subsection{Robustness to Watermark Removal Attacks}
When the user $u_i$ receives the watermarked image $x$, they can post-process it to obtain the other image, $x'$, which does not retain a sufficient part of the watermark $w(u_i)$. The transition from $x$ to $x'$ may be done by applying an image transformation, such as brightness or contrast adjustment, Gaussian blur, or additive noise. The other approach is to perform an adversarial attack on the generative model to erase the watermark \cite{jiang2024watermark}. 

\section{Spread them Apart: Pixel level}
\label{sec:method}

In this section, we provide a detailed description of the proposed approach, its implementation details, and the robustness guarantee against additive removal attacks of bounded magnitude. Note that in this section, we describe the embedding of the watermarks on the pixel level; an extension of the method is discussed in the subsequent sections. 
\subsection{Embedding and Extraction of the Watermark}
Suppose that  $f$ is the generative model. Recall that the user $u_i \in [u_1, \dots, u_m]$ is assigned a pair $(w(u_i), s(u_i))$ after the registration in the service, where both the watermark and the secret are unknown to the user and are privately kept by the owner of $f$. Let $x$ be the generated image. Then, the watermark injection and reconstruction are done as follows:

\begin{enumerate}
    \item The secret $s(u_i)$ is interpreted as two sequences of indices, $A = \{a_1, \dots, a_n\}$ and $B = \{b_1, \dots, b_n\}$. The watermark $w(u_i) = \{w_1, \dots, w_n\}$ is the binary string that restricts the generated image $x$ in the areas represented by the sets $A$ and $B$.
    \item The restriction of $x$ in the areas represented by the sets $A$ and $B$ given $w(u_i)$ is the following implication:
    \begin{equation}
    \label{eq:w_embed}
    \begin{cases}
        w_i = 0 \implies x_{a_i} \ge x_{b_i} \\
        w_i = 1 \implies x_{a_i} < x_{b_i},
    \end{cases}
      \end{equation}
      where $x_{j}$ is the intensity of the $j'$th pixel of $x$.  To increase the robustness  to watermark removal attacks, given $\varepsilon>0$, we apply additional regularization to $x$:
      \begin{equation}
          \min_{j \in [1,\dots,n]} |x_{a_j} - x_{b_j}| \ge \epsilon.
      \end{equation}
\end{enumerate}
To perform detection and attribution of the given image $x$, the owner of the generative model firstly constructs $m$ watermarks $w(u_1|x), \dots, w(u_m|x)$ by reversing the implication from the Eq. \ref{eq:w_embed}. Namely, given the secret $$s(u_i) = \{a_1, \dots, a_n, b_1, \dots, b_n\}$$ of user $u_i$, the watermark bits are restored by the following rule:
\begin{equation}
\begin{cases}
    x_{a_j} \ge x_{b_j} \implies w(u_i|x)_j = 0, \\
     x_{a_j} < x_{b_j} \implies w(u_i|x)_j = 1.     
\end{cases}
\end{equation}

\begin{remark}
Here, we distinguish the watermark $w(u_i)$ assigned by the owner of generative model to the user $u_i$ from the watermark $w(u_i|x)$ extracted from the image $x$ with the use of the secret $s(u_i)$ of user $u_i.$
\end{remark}
When $m$ watermarks $w(u_1|x), \dots, w(u_m|x)$ are extracted, the owner of the model assigns $x$ to the user $u$ with the minimum distance $d(w(u_i), w(u_i|x))$ between assigned and extracted watermarks:

\begin{equation}
\label{eq:attribution}
    u = \arg\min_{u_i \in [u_1, \dots, u_m]: \ \xi_i = 1} d(w(u_i), w(u_i|x)),
\end{equation}
where $\xi_i$ is the indicator function from the Eq. \ref{eq:indicator}. Note that if $\xi_i = 0$ for all $\ i \in [1,\dots,m],$ then $x$ is identified as image not generated by $f$. If the argument in Eq. \ref{eq:attribution} is not unique, the algorithm abstains from attribution.

\subsection{Implementation Details}

In this subsection, we describe the watermarking procedure. In our settings, the generative model $f$ is Stable Diffusion \cite{rombach2022high}. First of all, we have to note that in the Stable Diffusion model, the latent vector $z$ produced by the UNet is then decoded back into the image space using a VAE decoder, $\mathcal{D}$. To embed the watermark into an image, we optimize a special two-component loss function with respect to the latent vector $z.$ The overall loss is written as follows:
\begin{equation}
\label{eq:TotalLoss}
    \mathcal{L} = \lambda_{wm} \mathcal{L}_{wm} + \lambda_{qual} \mathcal{L}_{qual},
\end{equation}
The first term, $\mathcal{L}_{wm}$, defines how the image complies with the pixel difference imposed by the watermark $w(u_i) = \{w_1, \dots, w_n\}$ and the secret $s(u_i) = \{a_1, \dots, a_n, b_1, \dots, b_n\}$: 
\begin{equation}
\label{eq:wm_insert}
    \mathcal{L}_{wm} = \sum_{i=1}^{n}\min((-1)^{w_i} (x_{a_i} - x_{b_i}) + \varepsilon, 0), \quad x = \mathcal{D}(z).
\end{equation}
Here, $\varepsilon$ is the watermark robustness threshold. Note that the larger the value of $\varepsilon$ is, the more robust the watermark is to additive perturbations. 
The second term, $\mathcal{L}_{qual}$, is introduced to preserve the quality of the generated image. The value $\mathcal{L}_{qual}$ is a difference in image quality measured by the LPIPS metric \cite{zhang2018perceptual}, which acts as a regularization. 

The optimization is performed over $T= 700$ steps of the Adam optimizer with the learning rate of $8 \times 10^{-3}$, where every $100$ iteration, the learning rate is halved. When the convergence is reached, the ordinary Stable Diffusion post-processing of the image is performed. The coefficients $\lambda_{wm}$ and $\lambda_{qual}$ are determined experimentally and set to be $0.9$ and $150$, respectively, the value of $\varepsilon$ was set to be $\varepsilon=0.2$.  Schematically, the process of watermark embedding and extraction is presented in Figure \ref{fig:big_teaser}. 

\subsection{Robustness Guarantee}
The watermark embedded by our method is robust against additive watermark removal attacks of a bounded magnitude. Namely, let the watermark $w(u_i|x)$ be embedded in $x$ with the use of the secret $s(u_i) = \{a_1, \dots, a_n, b_1, \dots, b_n\}$ of the user $u_i$. Let  $\Delta_ i = \frac{|x_{a_i} - x_{b_i}|}{2}$.
Then, the following lemma holds.
\begin{lemma}\label{lemma:l_inf_lemma}
    Let $\varepsilon \in \mathbb{R}^d$ and $\Delta_{i_1} \le \Delta_{i_2} \le \dots \le \Delta_{i_n}$.     Then, if $\|\varepsilon\|_\infty < \Delta_{i_k}$, then $d(w(u_i|x + \varepsilon), w(u_i|x)) < k.$
\end{lemma}

The proof is presented in the supplementary material.  This lemma provides a lower bound on the $l_\infty$ norm of the additive perturbation $\varepsilon$ applied to $x$, which is able to erase at least $k$ bits of the watermark $w(u_i|x))$ embedded into $x$.

\section{Spread them Apart: Several Watermarks Instead of One}

Here, we discuss an extension of the proposed method to provide robustness to other types of input perturbations, such as rotations and translations. 
An extension is based on a simultaneous embedding of the watermark in the pixel space and special functions, which are invariant to certain transformations. An intuition behind this extension is as follows. Suppose that the transform $\gamma:\mathbb{R}^d \to \mathbb{R}^d$ is invariant under parametric perturbation $\phi: \mathbb{R}^d \times \Theta \to \mathbb{R}^d$, namely, $\gamma(x) =\gamma(\phi(x,\theta))$ for all $x \in \mathbb{R}^d,\ \theta \in \Theta$. Then, if the watermark is  successfully embedded into $\gamma(x)$, it becomes robust under perturbation $\phi$.

\subsection{Invariants in the Frequency Domain}
We will call image transform $\gamma:\mathbb{R}^d \to \mathbb{R}^d$ \emph{invariant} under parametric perturbation $\phi: \mathbb{R}^d \times \Theta \to \mathbb{R}^d$ at point $x\in \mathbb{R}^d$ if 
\begin{equation}
    \gamma(x) = \gamma(\phi(x, \theta)) \quad \text{for all}\quad  \theta \in \Theta,
\end{equation}
where $\Theta $ is the set of parameters of perturbation $\phi$. In this work, we use two invariants discussed in \cite{lin1993towards}, formulated as theorems below.

\begin{theorem}
    \label{th:trans_inv}
    Let $h(x,y)$ be an integrable nonnegative function, and its Fourier transform 
    \begin{align}
       H(\omega_x, \omega_y) = \int_{\infty}^\infty \int_{-\infty}^\infty h(x,y) e^{-i(x \omega_x +y\omega_y)}dxdy = A(\omega_x, \omega_y)e^{-i \psi(\omega_x, \omega_y)}
    \end{align}
     be twice differentiable. Then the function $A(\omega_x, \omega_y)$ is translation invariant.
\end{theorem}

\begin{theorem}
\label{th:rot_inv}
    Let $\tilde{h}(r,t) = h(e^r \cos t, e^r \sin t)$ be the logarithmic-polar representation of the image $h(x,y)$. The Fourier-Mellin transform of $\tilde{h}(r,t)$ is
    \begin{align}
        \tilde{H}(\omega, k) = \int_{-\infty}^\infty \int_0^{2\pi} \tilde{h}(r,t)e^{-i(kt +\omega r)}dt dr = \tilde{A}(\omega, k)e^{-i\tilde{\psi}(\omega,  k)},
    \end{align}
    where $\tilde{A}(\omega, k)$ is the magnitude and $\tilde{\psi}(\omega, k)$ is the phase. If $\tilde{h}(r,t)$ is an integrable nonnegative function and its Fourier-Mellin transform $\tilde{H}(\omega, k)$ is twice differentiable, then the function $\tilde{A}(\omega, k)$ is invariant under rotation. 
\end{theorem}
\begin{remark}
    We refer to invariants from Theorems \ref{th:trans_inv}-\ref{th:rot_inv} as to $\gamma_t$ and $\gamma_r$, respectively. 
\end{remark}

\subsection{Three Watermarks Instead of One}
Recall from Section \ref{sec:method} that the watermark embedding process in the pixel domain is done by optimizing the loss function $\mathcal{L}$ in the form from Eq. \ref{eq:TotalLoss}: $\mathcal{L} = \lambda_{wm}\mathcal{L}_{wm} + \lambda_{qual}\mathcal{L}_{qual},$ where 
\begin{equation}
    \mathcal{L}_{wm} = \sum_{i=1}^n \min((-1)^{w_i}(x_{a_i} - x_{b_i}) + \varepsilon, 0), \quad x = \mathcal{D}(z),
\end{equation}
$w(u_i) = \{w_1, \dots, w_n\}$ is the watermark assigned to user $u_i$ and $s(u_i)$ is the secret of user $u_i$. To ensure the robustness of the watermark to geometric transformations, we suggest embedding the watermark simultaneously in the pixel domain and in invariants $\gamma_t$ and $\gamma_r$. To do so, we optimize the loss function $\tilde{\mathcal{L}}$ in the form below:
\begin{align}
\label{eq:three_wms}
    & \tilde{\mathcal{L}} = \lambda_{wm}\mathcal{L}_{wm} + \lambda_{qual}\mathcal{L}_{qual} + \lambda_t \mathcal{L}_t + \lambda_r \mathcal{L}_r = \mathcal{L} + \lambda_t \mathcal{L}_t + \lambda_r \mathcal{L}_r.
\end{align}
In Eq. \ref{eq:three_wms}, $\lambda_t, \ \lambda_r$ are positive constants and 
$
\mathcal{L}_{p} = \sum_{i=1}^n \min((-1)^{w_i}(\gamma_p(x)_{a_i} - \gamma_p(x)_{b_i}) + \varepsilon, 0)$
is the loss function that controls the embedding of the watermark into invariants, $\gamma_t(x)$ for $p=t$ and $\gamma_r(x)$ for $p=r$, respectively.

\subsection{Extractions of Watermarks}

Given $m$ as the number of users,  the owner of the generative model extracts $3m$ watermarks from the given image $x$. Namely, given the secret $s(u_i)$ of the user $u_i$, three watermarks, $w(u_i|x), \ w(u_i|\gamma_r(x)), \ w(u_i|\gamma_t(x))$ are restored:
\begin{align}
    &\begin{cases}
    x_{a_j} \ge x_{b_j} \Longrightarrow w(u_i|x)_j = 0, \\
    x_{a_j} < x_{b_j} \Longrightarrow w(u_i|x)_j = 1,
    \end{cases} \\
    &\begin{cases}
          \gamma_p(x)_{a_j} \ge \gamma_p(x)_{b_j} \Longrightarrow w(u_i|\gamma_p(x))_j = 0, \\
          \gamma_p(x)_{a_j} < \gamma_p(x)_{b_j} \Longrightarrow w(u_i|\gamma_p(x))_j = 1, 
    \end{cases}
\end{align}
for $p =t$ and $p=r$ simultaneously. To assign the (possibly) watermarked image $x$ to the user, the owner of the model determines three candidates:
\begin{align}
   \begin{cases}
        u = \arg\min\limits_{u_i \in [u_1, \dots, u_m]: \xi_i=1} d(w(u_i), w(u_i|x)),\\
        u_p = \arg\min\limits_{u_i \in [u_1, \dots, u_m]: \xi^p_i=1} d(w(u_i), w(u_i|\gamma_p(x))), 
    \end{cases} 
\end{align}
where the variables 
\begin{equation}
    \xi^p_i = \mathds{1}[d(w(u_i|\gamma_p(x)), w(u_i)) \in [0,\tau_1] \cup [\tau_2, n]]
\end{equation}
indicate which of the users' watermarks are within an appropriate distance from the corresponding extracted watermarks, given the invariant $p$.

Finally, the owner of the model assigns the image $x$ to the user $\tilde{u}$ that corresponds to the minimum distance  among all $3m$  pairs of watermarks:
\begin{align}
    \tilde{u} = \arg\min & \{d(w(u), w(u|x)), d(w(u_r), w(u_r|\gamma_r(x))), d(w(u_t), w(u_t|\gamma_t(x)))\}. \nonumber
\end{align}

\subsection{Probability of Incorrect Attribution}
\label{sec:updated_fpr}

Assume that the user $u_i$ owns the watermarked image $x$. Note that the attribution of the image to the user $u_i$ is guaranteed to hold if $u=u_i, u_r=u_i, u_t=u_i.$ Hence, the probability of incorrect attribution, ${FPR}_3(m)$, is bounded from above by the sum 
    $\mathbb{P}(u \ne u_i) +  \mathbb{P}(u_r \ne u_i) + \mathbb{P}(u_t \ne u_i),$
yielding  $ FRP_3(m) \le 3\hat{p}$, where $\hat{p}$ is from Eq. \ref{eq:fpr_m}.

\section{Experiments}

\subsection{General Setup}
 We use \texttt{stable-diffusion-2-base}  model \cite{rombach2022high} with the \texttt{epsilon} prediction type and $50$ steps of denoising for the experiments. The resolution of generated images is $512 \times 512$. The experiments were conducted on \texttt{DiffusionDB} dataset \cite{wangDiffusionDBLargescalePrompt2022}. Specifically, we choose $1000$ unique prompts and generate $1000$ different images.  The public key for the user is sampled from the Bernoulli distribution with the parameter $p=0.5$. The length of a key is set to be $n=100$. The private key is generated by randomly picking $n$ unique pairs of indices of the flattened image.  


\subsection{Attack Details}\label{subseq:attack_details}

We evaluate the robustness of the watermarks embedded by our method against the following watermark removal attacks: brightness adjustment, contrast shift, gamma correction, image sharpening, hue adjustment, saturation adjustment, random additive noise, JPEG compression, and the white-box PGD adversarial attack  \cite{madry2018towards}. In this section, we describe these attacks in detail. 

\begin{itemize}
    \item Brightness adjustment of an image $x$ was performed by adding a constant value to each pixel: $x_{brightness} = x + b$, where $b$ was sampled from the uniform distribution $\mathcal{U}[-20, 20]$.
    
    \item Contrast shift was done in two ways: positive and negative. The positive contrast shift implies the multiplication of each pixel of an image by a constant positive factor: $x_{contrast} = c x$, where $c$ was sampled from the uniform distribution, $c \sim \mathcal{U}[0.5, 2]$. 

    \item On the contrary, when the contrast shift is performed with the negative value of $c$ (namely, $ c \sim \mathcal{U}[-2, -0.5]$), such a transform turns an image into a negative. Later, we treat these transforms separately and denote them as ``Contrast $+$'' and ``Contrast $-$'', depending on the sign of $c$. Note that ``Contrast $-$'' imitates a bit flipping attack.

    \item Gamma correction is nothing but taking the exponent of each pixel of the image: $x_{gamma} = x ^ g$, where $g \sim \mathcal{U}[0.5, 2]$.

    \item For sharpening, hue, and saturation adjustment, we use implementations from the \texttt{Kornia} package \cite{riba2020kornia} with the following parameters: $a_{saturation} = 2, \quad a_{hue} = 0.2,$  $\quad a_{sharpness} = 2.$

    \item The additive noise was sampled from the uniform distribution $\mathcal{U}[-\delta, \delta]$, where $\delta=25$. Note that the maximum $\|\cdot\|_\infty $ of noise is then equal to $25$.

    \item JPEG compression was performed using DiffJPEG \cite{Shin2017JPEGresistantAI} with quality equal to $50$. 

    \item White-box attack aims to change the embedded watermark $w$ to some other watermark $\hat{w}$ by optimizing the image with respect to the loss initially used to embed the watermark $w$:
\begin{align}
\label{eq:wb_attack_loss}
    &\mathcal{L}_{wb} = \lambda_{wm} \mathcal{L}_{wm} + \lambda_{qual} \mathcal{L}_{qual}, \quad \text{where} \\
    &\mathcal{L}_{wm} = \sum_{i=1}^{n}\min((-1)^{\hat{w}_i} (x_{a_i} - x_{b_i}) + \varepsilon, 0). \nonumber
\end{align}
In Eq. \ref{eq:wb_attack_loss}, the term $\mathcal{L}_{qual}$ corresponds to the difference in image quality in terms of the LPIPS metric, namely,
$   \mathcal{L}_{qual} = LPIPS(x, \hat{x})$
where $x$ and $\hat{x}$ are the original image and image on a particular optimization iteration, respectively.

 The loss function $\mathcal{L}_{wb}$ pushes the private key pixels to be aligned with a new randomly sampled public key $\tilde{w}$ so that the ground-truth watermark $w$ gets erased.  The attack's budget is the upper bound of $\|\cdot\|_\infty $ norm of the additive perturbation, which we have taken to be $ {\varepsilon}/{2}$ (see Eq. \ref{eq:wm_insert}). 
 
 If at some iteration the distance between the source image $x$ and the attacked one $\hat{x}$ exceeds $ {\varepsilon}/{2}$, then we project $\hat{x}$ back onto the sphere $\|\hat{x} - x\|_\infty = {\varepsilon}/{2}$. The optimization is performed for $10$ iterations with the Adam optimizer and the learning rate of $10^{-1}$. Note that this attack setting implies knowledge about the private key and assumes white-box access to the generative model. Hence, this is de facto the strongest watermark removal attack we consider. 
\end{itemize}
Pixels of the images perturbed by the attacks are then linearly mapped to $[0, 255]$ segment.

\subsection{Results: Pixel Level Watermarking}
\label{sec:main_results}
We report (i) bit-wise error of the watermark extraction caused by watermark removal attacks and (ii) True Positive Rates in attribution and detection problems. We compare our results to those of Stable Signature \cite{fernandez2023stable}, SSL watermarking \cite{fernandez2022watermarking}, AquaLora \cite{feng2024aqualora}, and WOUAF \cite{kim2024wouaf}, one of the state-of-the-art image watermarking approaches. In these works, the watermark length is set to be $48$, $30$, $48$, and $32$, respectively, while we have $100$ bits long watermarks; note that our method allows embedding of the longest  watermarks, among the considered competitors. A qualitative comparison of original and watermarked images can be found in the supplementary material. To evaluate the robustness of the watermarks against removal attacks, we report an average bit-wise error, ABWE:
\begin{equation}
    ABWE = \frac{1}{Nn}\sum\limits_{i=1}^{N}\sum\limits_{j=1}^{n}\mathds{1}[w^{gt}_{i,j}\neq w^{extracted}_{i,j}],
\end{equation}
where $w^{gt}_{i,j}$ and $ w^{extracted}_{i,j}$ are the $j$-th bits of ground truth and extracted watermarks, corresponding to the $i$-th image. Here, $n$ is the watermark length and $N=1000$ is the number of images. We report ABWE in Table \ref{table:attack-errors}.

\begin{table*}[!htb]
\begin{center}
\caption{Average bit-wise error after watermark removal attacks. The column ``Error'' corresponds to the average bit-wise error of the watermarking process.}
\begin{tabular}{cccccccccc}
\toprule
\multicolumn{1}{c}{Method}
&\multicolumn{1}{c}{Error}
&\multicolumn{1}{c}{Brightness}
&\multicolumn{1}{c}{Contrast $+$}
&\multicolumn{1}{c}{Contrast $-$}
&\multicolumn{1}{c}{Gamma}
&\multicolumn{1}{c}{JPEG}
\\ \midrule 
Ours   &$0.001$  &$\textbf{0.002}$    &$\textbf{0.002}$ &$\textbf{0.002}$ &$\textbf{0.003}$ &$0.147$ \\
Stable signature  &$0.012$    &$0.025$ &$0.025$ & $0.518$ & $0.016$ &$0.167$ \\
SSL  & $\textbf{0.000}$ & $0.117$ & $0.071$ & $0.785$& $0.041$ &${0.205}$ \\

AquaLora&  $ 0.043$ & $0.043$ & $0.043$ & $0.544$& $0.049$ &${0.056}$ \\
WOUAF & $0.031$ & $0.005$ & $0.008$ & $0.317$ & $0.004$ & $\textbf{0.014}$ \\

\midrule 
\multicolumn{1}{c}{Method}
&\multicolumn{1}{c}{Hue}
&\multicolumn{1}{c}{Saturation}
&\multicolumn{1}{c}{Sharpness}
&\multicolumn{1}{c}{Noise}
&\multicolumn{1}{c}{PGD}
\\ \midrule 
Ours     &$\textbf{0.010}$    &$0.110$ &$\textbf{0.001}$ &${0.056}$ &${0.064}$ & \\
Stable signature & $0.014$   &$\textbf{0.014}$ &$0.010$ & $0.136$ & $0.487$ & \\
SSL  & $0.061$ &$0.198$ &$0.003$ &$0.254$ &$0.272$ &  \\

AquaLora& $0.047$ & $0.059$ & $0.049$ & $0.073$  & $\textbf{0.052}$ \\
WOUAF & $0.010$ & $0.041$ & $0.005$ & $\textbf{0.005}$ & $0.132$ \\

\bottomrule
\end{tabular}

\label{table:attack-errors}
\end{center}
\end{table*}

\begin{table*}[!htb]
\begin{center}
\caption{TPRs under different watermark removal attacks, attribution problem.}
\begin{tabular}{ccccccccccc}
\toprule
\multicolumn{1}{c}{Method}
&\multicolumn{1}{c}{Hit}
&\multicolumn{1}{c}{Brightness}
&\multicolumn{1}{c}{Contrast $+$}
&\multicolumn{1}{c}{Contrast $-$}
&\multicolumn{1}{c}{Gamma}
&\multicolumn{1}{c}{JPEG}
\\ \midrule 
Ours &$\textbf{1.000}$   &$\textbf{1.000}$ &$\textbf{1.000}$ &$\textbf{1.000}$ &$\textbf{1.000}$ &$0.452$ \\
Stable signature &$0.983$   &$0.977$ &$0.976$ & $0.000$ & $0.982$ &$0.576$ \\
SSL & $1.000$ & $0.859$ & $0.905$ & $0.666$ & $0.996$ & $0.673$ \\
AquaLora &$0.955$   & $0.955$ & $0.953$ & $0.178$ & $0.940$ & ${0.936}$ \\
WOUAF &${0.999}$   &$0.997$ & $0.993$ & $0.002$ & $0.995$ &$\textbf{0.967}$ \\
 
 \midrule 
\multicolumn{1}{c}{Method}
&\multicolumn{1}{c}{Hue}
&\multicolumn{1}{c}{Saturation}
&\multicolumn{1}{c}{Sharpness}
&\multicolumn{1}{c}{Noise}
&\multicolumn{1}{c}{PGD}
\\ \midrule 
Ours &$\textbf{1.000}$    &${0.653}$ &$\textbf{1.000}$ &${0.971}$ &$\textbf{0.993}$ & \\
Stable signature & $0.982$   &$\textbf{0.982}$ & $0.983$ &$0.675$ &$0.000$ & \\
SSL & $0.987$ & $0.625$ & $0.996$ & $0.490$ & $0.449$\\
AquaLora & $0.949$   & ${0.914}$ & $0.941$ & $0.908$ & $0.947$ &  \\
WOUAF & $0.999$  & $0.862$ &$0.997$ &$\textbf{0.974}$ & $0.197$ & \\
\bottomrule
\end{tabular}

\label{table:collision}
\end{center}
\end{table*}

\begin{table*}[!htb]
\begin{center}
\caption{TPRs under different watermark removal attacks, detection problem.}
\begin{tabular}{ccccccccc}
\toprule
\multicolumn{1}{c}{Method}
&\multicolumn{1}{c}{Hit}
&\multicolumn{1}{c}{Brightness}
&\multicolumn{1}{c}{Contrast $+$}
&\multicolumn{1}{c}{Contrast $-$}
&\multicolumn{1}{c}{Gamma}
&\multicolumn{1}{c}{JPEG}
\\ \midrule 
Ours &$\textbf{1.000}$    &$\textbf{1.000}$ &$\textbf{1.000}$ &$\textbf{1.000}$ &$\textbf{1.000}$ &$0.452$ \\
Stable signature &${0.983}$   &$0.977$ &$0.976$ & $0.000$ & $0.982$ &$0.576$ \\
SSL  &$1.000$    &$0.863$ &$0.906$ & $0.674$ & $0.996$ &$ {0.681}$ \\
AquaLora &$0.955$   & $0.955$ & $0.953$ & $0.178$ & $0.940$ & ${0.936}$ \\
WOUAF &${0.999}$   &$0.997$ & $0.993$ & $0.002$ & $0.995$ &$\textbf{0.967}$ \\
\midrule 
\multicolumn{1}{c}{Method}
&\multicolumn{1}{c}{Hue}
&\multicolumn{1}{c}{Saturation}
&\multicolumn{1}{c}{Sharpness}
&\multicolumn{1}{c}{Noise}
&\multicolumn{1}{c}{PGD}
\\ \midrule 
Ours &$\textbf{1.000}$    &${0.653}$ &$\textbf{1.000}$ & ${0.971}$ &$\textbf{0.993}$ \\
Stable signature & $0.982$    & $ \textbf{0.982}$ & $0.993$  &$0.675$ &$0.000$ \\
SSL  &${0.987}$    & $0.628$ & $0.996$ & $0.495$ & $0.459$ \\
AquaLora & $0.949$   & ${0.914}$ & $0.941$ & $0.908$ & $0.947$ &  \\
WOUAF & $0.999$  & $0.862$ &$0.997$ &$\textbf{0.974}$ & $0.197$ & \\
\bottomrule
\end{tabular}

\label{table:orig-img-error}
\end{center}
\end{table*}

\begin{table*}[!htb]
\caption{TPRs under different watermark removal attacks, attribution problem.  }
\label{table:attribution_3wms}
\begin{center}
\begin{tabular}{ccccccccccc}
\toprule
\multicolumn{1}{c}{Method}
&\multicolumn{1}{c}{Hit}
&\multicolumn{1}{c}{Brightness}
&\multicolumn{1}{c}{Contrast $+$}
&\multicolumn{1}{c}{Contrast $-$}
&\multicolumn{1}{c}{Gamma}
&\multicolumn{1}{c}{JPEG}
&\multicolumn{1}{c}{Translation}
\\ \midrule 
STA(1) &${0.972}$   &${0.928}$ &${0.928}$ &${0.928}$ &${0.928}$ &$\textbf{0.724}$ & $0.000$\\
STA(3) &$\textbf{0.980}$   &$\textbf{1.000}$ & $\textbf{1.000}$ & $\textbf{1.000}$ &$\textbf{0.951}$ & ${0.563}$  & $\textbf{0.962}$\\
 
 \midrule 
\multicolumn{1}{c}{Method}
&\multicolumn{1}{c}{Hue}
&\multicolumn{1}{c}{Saturation}
&\multicolumn{1}{c}{Sharpness}
&\multicolumn{1}{c}{Noise}
&\multicolumn{1}{c}{PGD}
&\multicolumn{1}{c}{Rotation}
\\ \midrule 
STA(1) &$\textbf{1.000}$    &${0.571}$ &$\textbf{1.000}$ &$\textbf{0.928}$ &$\textbf{0.857}$ & ${0.000}$\\
STA(3) & $\textbf{1.000}$ & $\textbf{1.000} $   &$\textbf{1.000}$ & \textbf{0.928} &$\textbf{0.857}$ &$\textbf{0.489}$ & \\
\bottomrule
\end{tabular}
\end{center}
\end{table*}

\begin{table*}[!htb]
\caption{TPRs under different watermark removal attacks, detection problem. }
\label{table:detection_3wm}
\begin{center}
\begin{tabular}{cccccccccc}
\toprule
\multicolumn{1}{c}{Method}
&\multicolumn{1}{c}{Hit}
&\multicolumn{1}{c}{Brightness}
&\multicolumn{1}{c}{Contrast $+$}
&\multicolumn{1}{c}{Contrast $-$}
&\multicolumn{1}{c}{Gamma}
&\multicolumn{1}{c}{JPEG}
&\multicolumn{1}{c}{Translation}
\\ \midrule 
STA(1) &${0.972}$   &${0.928}$ &${0.928}$ &${0.928}$ &${0.928}$ &$\textbf{0.724}$ & $0.000$\\
STA(3) &$\textbf{0.980}$   &$\textbf{1.000}$ & $\textbf{1.000}$ & $\textbf{1.000}$ &$\textbf{0.951}$ & ${0.563}$  & $\textbf{0.962}$\\
 
 \midrule 
\multicolumn{1}{c}{Method}
&\multicolumn{1}{c}{Hue}
&\multicolumn{1}{c}{Saturation}
&\multicolumn{1}{c}{Sharpness}
&\multicolumn{1}{c}{Noise}
&\multicolumn{1}{c}{PGD}
&\multicolumn{1}{c}{Rotation}
\\ \midrule 
STA(1) &$\textbf{1.000}$    &${0.571}$ &$\textbf{1.000}$ &$\textbf{0.928}$ &$\textbf{0.857}$ & ${0.000}$\\
STA(3) & $\textbf{1.000}$ & $\textbf{1.000} $   &$\textbf{1.000}$ & \textbf{0.928} &$\textbf{0.857}$ &$\textbf{0.489}$ & \\
\bottomrule
\end{tabular}
\end{center}
\end{table*}



To estimate the TPR in the attribution problem, we extract $m=10000$ different watermarks from the watermarked images. To extract a different watermark, we randomly generate $m=10000$ different private keys to simulate other users. The results are reported in Table \ref{table:collision} together with the TPRs under different watermark removal attacks. Note that the PGD attack in this setting is aimed at restoring the original image by removing the watermark.  To estimate the TRP in the watermark detection problem, we do the same procedure for non-watermarked images generated by the Stable Diffusion model and extract $m=10000$ different watermarks. We fix FPR = $10^{-6}$; such a FPR is achieved when $\tau_1 = 25$ and $\tau_2 = 75$ from Eq. \ref{eq:indicator}. The results are presented in Table \ref{table:orig-img-error}. 

Note that our framework yields both low misattribution and misdetection rates according to the two-tailed detection and attribution rules from the Eq. \ref{eq:attribution}. The proposed approach yields watermarks that are provably robust to additive perturbations of a bounded magnitude, multiplicative perturbations of any kind, and exponentiation.




\subsection{Results: Several Watermarks Instead of One}

In Tables \ref{table:attribution_3wms}-\ref{table:detection_3wm}, we include rotation and translation transformations of the image. 
To rotate an image, we sample an angle of rotation $\theta_r$ uniformly from $[-10^{\circ},10^{\circ}]$;  to translate an image, we sample the vector of translation $(\theta^x_t, \theta^y_t)$ uniformly from the set $[-10^{\circ}, 10^{\circ}] \times [-10^{\circ},10^{\circ}]$. We compare the results of an extended version of the proposed method with the baseline approach. STA(1) refers to the baseline approach described in the manuscript, STA(3) refers to an extended approach. We use $m=10000$ different private keys and fix FPR = $10^{-6}$. Note that to achieve FPR  $= 10^{-6}$ for STA(3), we set $\tau_1=24$ and $\tau_2=76$. In the column ``Hit'', we report the bit-wise accuracy of the watermarking procedure. To focus on the evaluation of the robustness of the watermarks, we report results for the images in which the watermark is embedded successfully. Namely, after embedding the watermark, we choose the watermarked images with a generation error smaller than $0.05$. It is noteworthy that the simultaneous embedding of the watermark into the pixel domain and corresponding invariants in the frequency domain significantly improves the robustness against  watermark removal attacks. 

\section{Conclusion}
In this paper, we propose \emph{Spread them Apart}, a framework to watermark the generated content of continuous nature and apply it to images generated by Stable Diffusion. We show theoretically that the watermarks produced by our method are provably robust against additive watermark removal attacks of a bounded norm and are provably robust to multiplicative perturbations by design. Our approach can be used to both detect that the image is generated by a given model and to identify the end-user who generated it. Experimentally, we show that our method is comparable to the state-of-the-art watermarking methods in terms of the robustness to synthetic watermark removal attacks.

\bibliographystyle{splncs04}
\bibliography{main}
\end{document}